\RequirePackage{fix-cm}
\documentclass[smallcondensed]{svjour3}     
\smartqed  
\usepackage{graphicx}
\usepackage{algorithm}
\usepackage{algorithmic}

\newcommand{\Sec}[1]{\hyperref[sec:#1]{\S\ref*{sec:#1}}} 
\newcommand{\Eqn}[1]{\hyperref[eq:#1]{(\ref*{eq:#1})}} 
\newcommand{\Fig}[1]{\hyperref[fig:#1]{Figure~\ref*{fig:#1}}} 
\newcommand{\Tab}[1]{\hyperref[tab:#1]{Table~\ref*{tab:#1}}} 
\newcommand{\Thm}[1]{\hyperref[thm:#1]{Theorem~\ref*{thm:#1}}} 
\newcommand{\Lem}[1]{\hyperref[lem:#1]{Lemma~\ref*{lem:#1}}} 
\newcommand{\Prop}[1]{\hyperref[prop:#1]{Property~\ref*{prop:#1}}} 
\newcommand{\Cor}[1]{\hyperref[cor:#1]{Corollary~\ref*{cor:#1}}} 
\newcommand{\Def}[1]{\hyperref[def:#1]{Definition~\ref*{def:#1}}} 
\newcommand{\Alg}[1]{\hyperref[alg:#1]{Algorithm~\ref*{alg:#1}}} 
\newcommand{\Ex}[1]{\hyperref[ex:#1]{Example~\ref*{ex:#1}}} 

\newcommand{\Real}{\mathbb{R}}

\newcommand{\Fold}[1]{{\tt fold}\left( #1 \right)}
\newcommand{\Unfold}[1]{{\tt unfold}\left( #1 \right)}

\newcommand{\bcirc}{ {\tt bcirc}}

\newcommand{\V}[1]{{\bm{\mathbf{\MakeLowercase{#1}}}}} 

\newcommand{\M}[1]{{\bm{\mathbf{\MakeUppercase{#1}}}}} 

\newcommand{\T}[1]{\boldsymbol{\mathscr{\MakeUppercase{#1}}}} 
\newcommand{\ormat}[1]{\overrightarrow{\mathcal{#1}}} 
\newcommand{\tube}[1]{\overrightarrow{\bold{#1}}} 

\newcommand{\TA}{\T{A}}
\newcommand{\TB}{\T{B}}

\newcommand{\TV}{\T{V}}

\newcommand{\fft}{ \mbox{\tt fft} }
\newcommand{\ifft}{ \mbox{\tt ifft} }

\usepackage{bm,color}
\usepackage[mathscr]{eucal}

\usepackage{graphicx,wrapfig} 
\usepackage{amsmath,amssymb}
\usepackage{amsbsy}

\usepackage{hyperref}
\usepackage{url}

\usepackage{fixltx2e}
\MakeRobust{\overrightarrow}

\newcommand{\TY}{\T{Y}}

\newcommand{\TW}{\T{W}}

\newcommand{\argmin}{{\text{argmin}}}

\begin{document}

\title{Clustering multi-way data: a novel algebraic approach \thanks{This research is funded in part by NSF:1319653. The first author acknowledges support for this work by the Tufts Summer Scholars's Program 2013.}
}


\author{Eric Kernfeld         \and
        Shuchin Aeron \and
        Misha Kilmer  
}


\institute{Eric Kernfeld \at
              Dept. of Statistics, University of Washington, Seattle, WA \\
              \email{fauthor@example.com}           
           \and
           Shuchin Aeron \at
              Dept. of Electrical and Computer Engineering, Tufts University, Medford, MA
              \email{shuchin@ece.tufts.edu}
              \and
             Misha Kilmer \at
              Dept. of Mathematics, Tufts University, Medford, MA
              \email{Misha.Kilmer@tufts.edu}
}

\date{Received: date / Accepted: date}

\maketitle

\begin{abstract}
In this paper, we develop a method for unsupervised clustering of two-way (matrix) data by combining two recent innovations from different fields: the Sparse Subspace Clustering (SSC) algorithm \cite{Elhamifar:0zr}, which groups points coming from a union of subspaces into their respective subspaces, and the t-product \cite{Kilmer:2011vn}, which was introduced to provide a matrix-like multiplication for third order tensors. Our algorithm is analogous to SSC in that an ``affinity'' between different data points is built using a sparse self-representation of the data. Unlike SSC, we employ the t-product in the self-representation. This allows us more flexibility in modeling; in fact, SSC is a special case of our method. 

When using the t-product, three-way arrays are treated as matrices whose elements (scalars) are n-tuples or tubes. Convolutions take the place of scalar multiplication. This framework allows us to embed the 2-D data into a vector-space-like structure called a free module over a commutative ring. These free modules retain many properties of complex inner-product spaces, and we leverage that to provide theoretical guarantees on our algorithm. We show that compared to vector-space counterparts, SSmC achieves higher accuracy and better able to cluster data with less preprocessing in some image clustering problems. In particular we show the performance of the proposed method on Weizmann face database, the Extended Yale B Face database and the MNIST handwritten digits database.
\end{abstract}

\section{Introduction}
In most clustering algorithms, the objects are assumed to be embedded in a normed linear space, and similarity is measured by a distance-like function. Among many different options, this can be followed by construction of a weighted graph in which similar points are joined by strong edges. Then, using tools from spectral graph theory, it is possible to find the clusters as connected graph components \cite{njwspectral}
. 

In contrast, subspace clustering techniques take a different perspective on the geometry of clustering. In these approaches, points are assumed to come from a union of subspaces rather than from disjoint, volume-filling clusters. For subspace clustering methods using spectral graph theory as a final step (which many do), any relevant notion of similarity should reflect whether points belong to the same subspace. Simple variations on distance are no longer effective. For this problem, a diverse array of subspace clustering methods \cite{Vidal:2011bh}, such as Sparse Subspace Clustering (SSC), can be employed to resolve the clusters. Some even reject outliers \cite{Elhamifar:0zr,HeckelB13_ISIT,Heckel_ArXiv13,Soltanolkotabi2013}. 


For clustering data with two-way structure, such as images, typical subspace clustering methods must unfold the data or map it to a vector. This approach sacrifices the two-way structure, potentially failing to exploit useful information. Outside of subspace clustering, on the other hand, many methods take advantage of multi-way array structure, particularly for dimensionality reduction or finding the best (single) subspace in which to approximate the data. For examples of these, see \cite{HeCN05_NIPS,Hao:2013uq,De-Lathauwer:2000uq,Sorber:2013ly} and references therein.

We know of no work exploiting multi-way structure in techniques similar to subspace clustering, and our goal is to fill that gap. In this paper we present a novel algebraic approach for clustering multi-way data. Whereas existing subspace clustering methods concatenate data as columns of a matrix, our method will group the data into a three-way array (a tensor), clustering slices of the tensor. Using the tensor products and factorizations outlined in \cite{Kilmer:2011vn}, we will extend Sparse Subspace Clustering \cite{Elhamifar:0zr} for use on tensor data. Although our strategy could process $N$-way ($N >2$) data by incorporating technical tools in the style of \cite{Martin:2013kx}, we chose to focus this work on clustering only two-way data. We will demonstrate that this new model is able to achieve higher accuracies than previous solutions, especially for data that has undergone less preprocessing.

Before we begin with necessary background in Section \ref{sec:back} we would like to summarize our main contributions below.
\paragraph{Our contributions:}
First, we propose a new algebraic generative model, based on a characterization of third order tensors from \cite{Braman2010} as operators, via multiplication called the t-product introduced in \cite{KilmerMartin2011}, on a 
space of oriented matrices.  
This model is explained in Sections \ref{subsec:free} and \ref{subsec:union_submod}.  
For inference with our model, we propose a novel clustering algorithm in Section \ref{sec:SSmC_Alg}. To characterize the algorithm's performance, in Section \ref{sec:SSmC_Theorems}, in this paper we add to the mathematical framework given in \cite{Braman2010,KilmerBramanHooverHao2012}.  Our new constructs then are used to derive performance bounds in Section \ref{sec:SSmC_Theorems} .  
Our completely new results are of similar flavor as those in \cite{Elhamifar:0zr} in that the ability to separate (cluster) data is characterized in terms of worst case \emph{tubal-angles} between submodules. However, these generalizations of the geometrical notions in \cite{Elhamifar:0zr} for linear algebra to the tensor and t-product framework are not immediately obvious, and the key technical development in this paper which makes this possible is stated in the theoretical results of the present manuscript and in the supplementary material. In Section \ref{sec:Sim}, we conduct experiments with synthetic data, the Weizmann Face data base \cite{Weizman}, the Extended Yale B Face Database \cite{Yale_DB1,Yale_DB2}, and the MNIST handwritten digits \cite{MNIST}.

\section{Background and Preliminaries}
\label{sec:back}

\subsection{Sparse Subspace Clustering}
Sparse Subspace Clustering (SSC) \cite{Elhamifar:0zr} is a recent method for solving the subspace clustering problem. Like many clustering methods, SSC constructs an affinity matrix whose $(i,j)$ entry is designed to be large when data points $i$ and $j$ are in the same subspace and relatively small or zero otherwise. To construct the affinity matrix, Sparse Subspace Clustering makes use of the fact that each datum can be expressed most efficiently {\it as a linear combination of members of its own subspace}. In particular, each data point is expressed as a linear combination of the others, with an $\ell_1$ regularization term to promote sparsity, \cite{Elhamifar:0zr}, \cite{GeometricSSC}. Spectral clustering \cite{Luxburg:2007dq} is then used to segment the data using the sparse coefficients as the affinity matrix. Variants of SSC have been considered in \cite{Soltanolkotabi2013} and \cite{Wang2013}, which guarantee correctness for SSC with bounded noise. Similar algorithms include low-rank representation (LRR) \cite{2010arXiv1010.2955L}, which uses a nuclear norm penalty in place of the $\ell_1$ regularization and also has theoretical guarantees on performance.

In this paper, we wish instead to promote the idea of clustering matrix objects by maintaining those objects in their
two-dimensional form, as opposed to vectorizing the matrices and clustering their vector representations.   Of course, a linear combination of vectors of length $mn$, if those vectors are reshaped into an $m \times n$ matrix, is a linear
combination of matrices.  It is important to know that we are proposing a far different representation for our clustering approach.   If we orient a $m \times n$ matrix by twisting it into the page, the resulting object is a $m \times 1 \times n$ third order tensor.  
But an $m \times 1 \times n$ tensor is a vector of length $m$, where each entry is an $1 \times 1 \times n$ tube
fiber.  Thus, the elemental objects are themselves these $1 \times 1 \times n$ tubes (called tube fibers in the tensor literature) in our scenario.   If we have a method of multiplying two tube fibers, then we have a method for writing ``linear'' combinations of oriented matrices, where our {\it weights are tube fibers, and not scalars}.   The method we will employ for this multiplication is the t-product, as specified in \cite{KilmerMartin2011,Braman2010,KilmerBramanHooverHao2012}.   As this product is central to the development of our algorithm and associated theoretical constructs, we now introduce the method.   

\subsection{The t-product}

We will now introduce the $t$-product \cite{Kilmer:2011vn,Braman2010} along with some notation. We will denote scalars with regular uppercase and lowercase font: $d_i, D\in \mathbb{R}$. Vectors will be in lowercase and bold, and matrices will be bold and uppercase, for example $\V{v}\in \mathbb{R}^n; \M{X} \in \mathbb{R}^{m\times n}$. In the following we will rely heavily on the MATLAB notation for indexing the elements of a tensor; we treat tensors as multiway arrays stored in MATLAB. For readers unfamiliar with MATLAB ``colon notation'', a colon in place of an index indicates that an entire cross-section of the array is being accessed; for example, column one of $\M{X}$ is $\M{X}(:,1)$. We will also use the MATLAB function $\fft$ and its standard usage. On a three-way array, $\fft (\T{X},[\,],3)$ applies the transform the the third dimension of the array, that is, to each tube fiber $\T{X}(i,j,:)$ separately.

We view a 3-D tensor $\T{X} \in \Real^{H \times L \times D}$ as an $H \times L$ matrix of tubes in $\Real^{1\times1\times D}$. These tubes will be denoted by adding an arrow over the vector notation: $\tube{\V{v}} \in \Real^{1\times1\times D}$. 
Similarly, we think of a $H \times 1\times D$ tensor as a vector of tubes. We will call such tensors ``oriented matrices'' (see Figure \ref{fig:t_prod}) and denote them with arrows and calligraphic script: $\ormat{M}\in \Real^{H \times 1\times D}$. In order to 
define matrix-like operations, we multiply these tubes (in our approach, tubes take the place of scalars) using a commutative operation between two tubes $\tube{\V{v}} \in \Real^{1\times 1 \times D}$ and $\tube{\V{u}}  \in \Real^{1\times 1 \times D}$ resulting in another tube of same length. The commutative operation used is the circular convolution; we will write $\ast$ to denote the circular convolution. 

\begin{figure}[htbp]	
\centering
 	\includegraphics[height=1.5in]{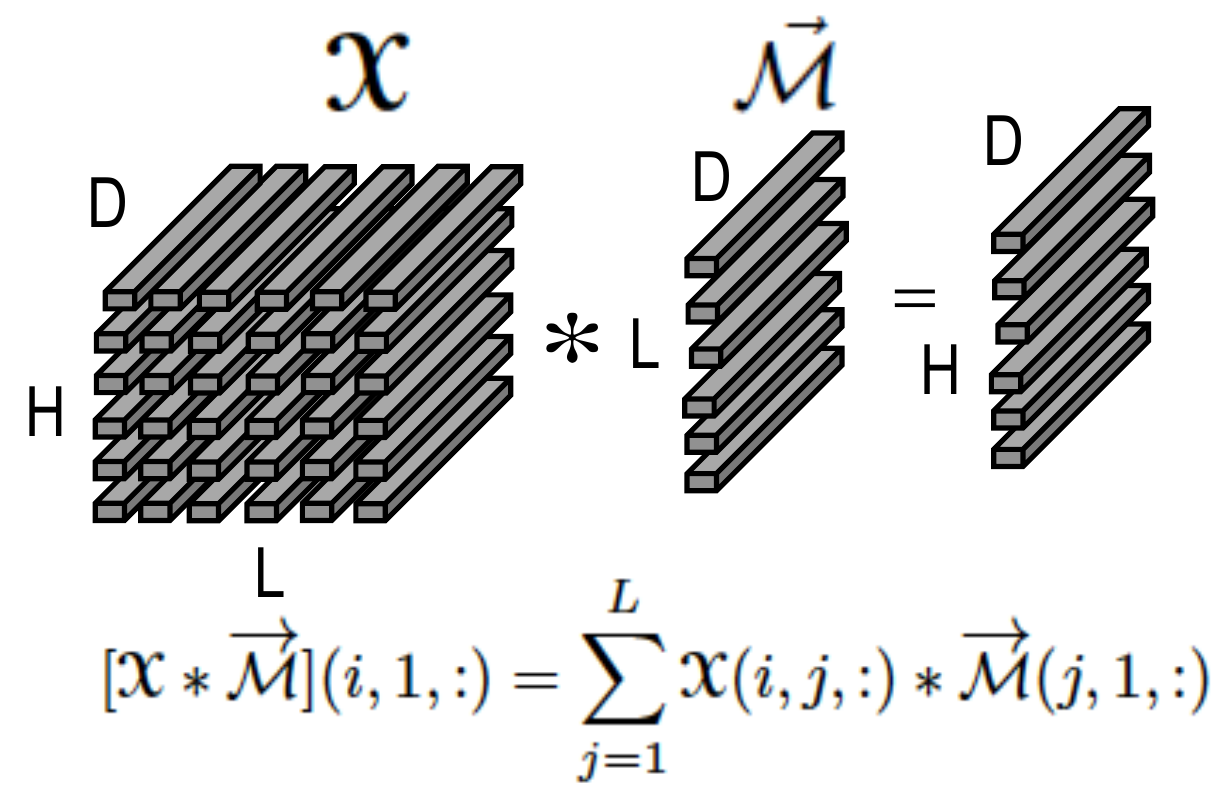}
 	\caption{\label{fig:t_prod} \footnotesize{An illustration of a 3-D tensor acting as an operator on an oriented matrix, $\ormat{M}$.  The result is oriented matrix. }}
	\end{figure} 

Under this construction, the product of a tensor $\T{X} \in \Real^{H \times L \times D}$ with an oriented matrix $\ormat{M} \in \Real^{L \times 1 \times D}$ is another oriented matrix of size $H \times 1 \times  D$ whose $i$-th tubal element given by, $$ [\T{X} \ast \ormat{M}](i,1,:) = \sum_{j=1}^{L} \T{X}(i,j,:) \ast \ormat{M}(j,1,:)$$ as illustrated in Figure~\ref{fig:t_prod}. Similarly one can extend this definition to define the multiplication of two tensors $ \T{X}$ and $\T{Y}$ of sizes $H \times L \times D$ and $L \times K \times D$ respectively, resulting in a tensor $\T{C} = \T{X} \ast \T{Y}$ of size $H \times K \times D$. This product between two tensors is referred to as the t-product, and its properties are developed in \cite{Kilmer:2011vn,Braman2010,KilmerBramanHooverHao2012}. 

\paragraph{Computing the t-product using the FFT}
\label{sec:tprod_fft}
The t-product can be effectively computed using the Fast Fourier Transform. Using MATLAB notation and built-in functions $\fft$ and $\ifft$, two tensors can be multiplied using the following steps (See \cite{Kilmer:2011vn} for proofs):
\begin{enumerate}
\item Compute Fourier transform along the 3-rd dimension - $\hat{\T{X}} = \fft (\T{X},[\,],3), \hat{\T{Y}} = \fft (\T{Y},[\,],3)$
\item Compute face-product in the Fourier domain - $ \hat{\T{C}} = \hat{\T{X}} \odot \hat{\T{Y}}$, where the $d$-th frontal face is given by $\hat{\T{C}}(:,:,d) = \hat{\T{X}}(:,:,d) \hat{\T{Y}}(:,:,d)$.

\item Take the inverse Fourier transform along the 3-rd dimension - $\T{C} = \T{X} \ast \T{Y} =  \ifft (\hat{\T{X}} \odot \hat{\T{Y}},[\,],3)$.
\end{enumerate}

\subsubsection{Linear algebra with the t-product: notation and facts}
\label{subsec:free}

As noted above, our elemental objects are tubes in $\mathbb{K}_D$, rather than scalars in $\mathbb{C}$.  
Now $\mathbb{C}$ with standard scalar addition and multiplication forms what is referred to in abstract algebra as a field.   And $\mathbb{C}^n$ equipped over this field forms a vector space.   
But $\mathbb{K}_D$ equipped with $*$ does not form a field, because there are non-zero tubes which 
are not invertible.   However, $(\mathbb{K}_D, *)$ does form what is referred to as a ring with unity \cite{KilmerMartin2011}.   
A module over a ring can be thought of as a generalization of the concept of a vector space over a field, where the corresponding scalars are now the elements of the ring.  
In linear algebra over a ring, {\it the analog of a subspace is a free submodule}.   Our algorithm relies on 
submodules, and as such, we need to carefully set up the rest of the framework.   

We begin this section by presenting a theorem from \cite{Braman2010}, because it and its corollaries imply that many useful properties of subspaces are still present with free submodules. In a variation on the notation in that work, we denote the set of length $D$ tubes by $\mathbb{K}_D$ and the set of oriented matrices of size $H  \times D$ by $ \mathbb{K}_D^{H}$. Likewise, $\mathbb{K}_{D}^{L\times H}$ represents the set of tensors of size $H \times L \times D$. (We use $\mathbb{K}_{D}$ in place of $\Real^{1\times1\times D}$ to signal that the algebraic structure associated with the set is the t-product rather than the more widespread framework of multilinear algebra.) We note that the set $\mathbb{K}_D$ forms a ring with identity using multiplication given by $\ast$ and the usual addition. Then we have the theorem from \cite{Braman2010} is as follows. 

\begin{theorem}[Braman \cite{Braman2010}]
\label{thm:free}
The set of oriented matrices $ \mathbb{K}_D^{H}$ forms a free module over $\mathbb{K}_D$.
\end{theorem}

The analogue of a subspace is that of a submodule as defined below. 

\begin{definition} A set  $\mathbb{S} \subset  \mathbb{K}_D^{H}$ is a submodule of $\mathbb{K}_D^{H}$ if it is a subset of $\mathbb{K}_D^{H}$, it contains the $0$ element in $\mathbb{K}_D^H$, is 
closed under addition\footnote{Addition in this set is component-wise.} and is closed under multiplication with a tube under t-product multiplication. \end{definition}

Just as a subspace has a basis, a free submodule has a generating set so that any element of the free submodule can be written as a ``t-linear combination" of generating elements \footnote{The adjective ``free'' applies when the submodule has a generating set.} . By ``t-linear combination,'' we mean a sum of oriented matrices multiplied, with the t-product, by coefficients from $\mathbb{K}_D$; we illustrate this in Figure~\ref{fig:t_lin}. 

Furthermore, free submodules have the unique dimension property: if two generating sets are ``linearly independent'' (we will define linear independence momentarily in Definition \ref{def:lin_ind}), they will be the same size, and this size is what we call as the submodular dimension. 

\begin{figure}[htbp]
\centering \includegraphics[scale=0.6]{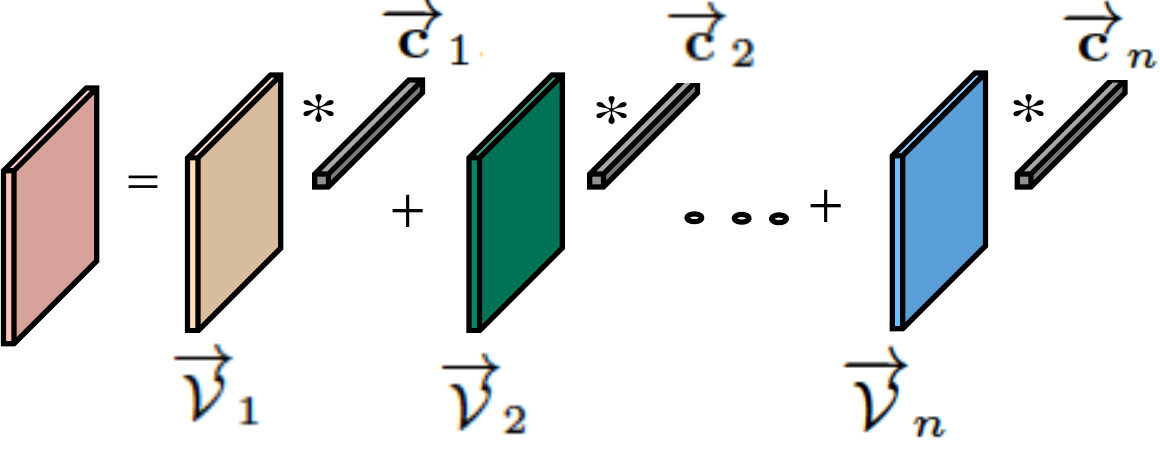}
\caption{\footnotesize{An element of a submodule generated by t-linear combinations of oriented matrices $\ormat{V}_1,\ormat{V}_2, \hdots, \ormat{V}_n$ with coefficients $\tube{\V{c}}_1, \tube{\V{c}}_2, \hdots, \tube{\V{c}}_n$.}}
\label{fig:t_lin} 
\end{figure}

We now give a notion of transpose for third order tensors. This definition is from \cite{Kilmer:2011vn} where it is shown that the new transpose preserves the following property $(\TA\ast\TB)^\top=(\TB^\top \ast \TA^\top)$.
\begin{definition}
\emph{\textbf{Tensor Transpose}}.  Let $\T{X}$ be a tensor of size $n_1 \times n_2 \times n_3$. Then $\T{X}^\top$ is the $n_2 \times n_1 \times n_3$ tensor obtained by transposing each of the frontal slices and then reversing the order of transposed frontal slices $2$ through $n_3$.
\end{definition}

Next, we will define the mathematically rigorous ideas of sum, linear independence, and disjointness for submodules which are useful for subsequent analysis.   {\it With the exception of linear independence, these have not been previously definited elsewhere in the literature on the t-product.  } In the case of scalars with depth one $D=1$, our definitions will reduce to the linear algebraic definitions as seen in \cite{Vidal:2011bh}, since convolution will reduce to scalar multiplication in that case. We also define linear independence for lists of oriented matrices. 
In this article, we will consider only free submodules.

\begin{definition}
\label{def:submod_sum}
The sum of two submodules $\mathbb{S}_i$ and $\mathbb{S}_j$ is defined as 
$\mathbb{S}_i {+}\mathbb{S}_j = \{\ormat{X}|\exists \ormat{V}\in \mathbb{S}_i,\ormat{W}\in\mathbb{S}_j $ so that $\ormat{X}=\ormat{V}+\ormat{W}\} $.\label{sum}
\end{definition}

\begin{definition}
 A collection of submodules is linearly independent if each one intersects the sum of the others only at the zero element of $\mathbb{K}_D^H$.
\end{definition}

\begin{definition}
 A collection of submodules is disjoint if each one intersects the union of the others only at the zero element of $\mathbb{K}_D^H$.
\end{definition}

 \begin{definition}
Consider the equation $\sum_{i=1}^N \ormat{V}_i \ast \tube{c}_i = \ormat{O}$, where the right hand side is the
oriented matrix of all zeros.   
Assume that the only solution is for all the $\tube{c}_i$ to be the zero tube.  Then, the collection of oriented matrices $\ormat{V}_1...\ormat{V}_N$ is linearly independent. \label{def:lin_ind}
\end{definition}

\section{Generative model: Union of Submodules}
\label{subsec:union_submod}

In this paper we will model the 2-D data as coming from union of free submodules and derive an algorithm in Section \ref{sec:SSmC_Alg} for clustering data. Before that we would like to explain some nuances of the proposed model. 

\textbf{Explaining t-linear combinations.}
The definition of the t-product as introduced in \cite{KilmerMartin2011} was written in terms of block circulant matrices, and we include it here to help interpret our model. Let $\T{A}^{(d)}$ denote the $d$-th frontal slice of $\T{A}$.   Using functions 
$$\bcirc(\TA)=\left[\begin{array}{llll} 
\TA^{(1)}&\TA^{(D)} &\hdots &\TA^{(2)}\\
\vdots&\TA^{(1)}&\ddots&\vdots\\
\TA^{(D-1)}&\vdots&\ddots&\TA^{(D)}\\
\TA^{(D)}&\TA^{(D-1)}&\hdots&\TA^{(1)}\end{array}\right],$$
and
$$ \Unfold \TA=\left[\begin{array}{c}\TA^{(1)}\\ \vdots \\ \TA^{(D)}\end{array}\right]$$, and $\Fold{}$ to invert $\Unfold{}$, the t-product $\TA\ast\TB$ was given by
 $$\TA * \TB = \Fold{\bcirc(\TA)\Unfold \TB}.$$

\begin{figure}[htbp]
\begin{center}
\includegraphics[scale=.25]{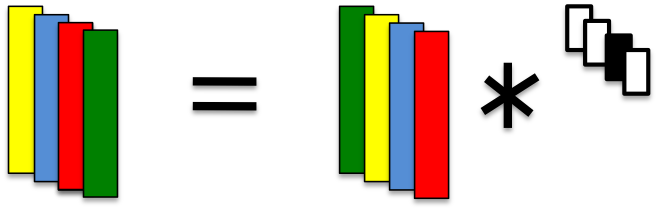}\\
\includegraphics[scale=.25]{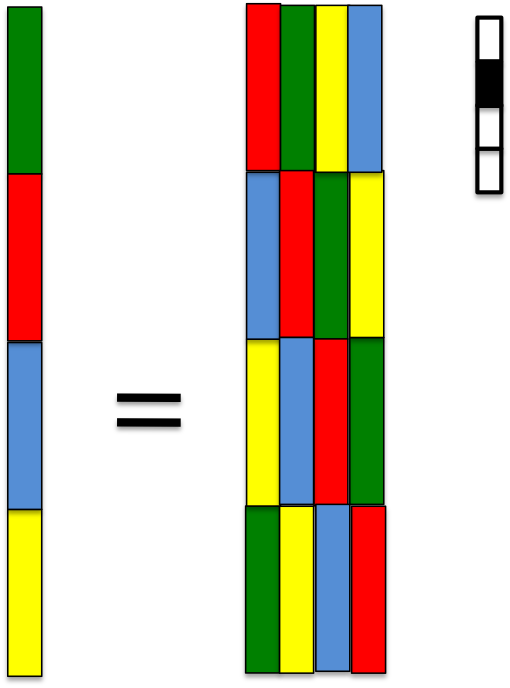}
\caption{\footnotesize{Using the t-product, data may be generated from shifted copies of generating elements.}}\label{fig:shifting}
\end{center}
\end{figure}

 Given a t-linear combination like the one in Figure \ref{fig:t_lin}, Figure \ref{fig:shifting} focuses in on a single oriented matrix and its ``scalar'' coefficient. For illustrative purposes, the coefficient is shown with only a single nonzero, in position 2. This nonzero element ``selects'' the second column of the block circulant matrix, which is then folded back into an oriented matrix. After folding, we observe that the new ``scalar multiple'' has been shifted relative to the original.

\subsection{Submodules in place of subspaces} The takeaway from the demo above is that every oriented matrix spans a subspace containing multiples only of itself, but generates a submodule containing multiples of certain permutations of its column fibers as well. The presence of these shifted copies is what distinguishes t-linear combinations from linear combinations. To use terms from signal processing, the t-linear combination uses each coefficient tube to represent a filter. For our proposed approach to succeed, the structure imposed by the filter must capture the patterns in the data, and we will argue below that the block circulant matrices used here are suitable to represent data sets generated from images.

We note that of course one could construct a pathological case for which this model is not adequate by creating two clusters, one consisting of shifted copies of the other's elements. These would be distinguishable by subspace clustering, but not submodule clustering. Even though these pathological examples exist in theory, we believe the shifting of image columns that we allow is particularly suited to images, since many causes of within-cluster variation--for instance, moving subjects or camera panning--can be approximately represented by shifts while between-cluster variation usually does not resemble shifting.

Modeling image collections with submodules is further warranted by its success in previous applications such as video restoration from missing pixels \cite{ZhangEAHK13} and face recognition \cite{Hao:2013uq}. This prior work shows that the low-dimensional-submodule assumption can provide a useful framework to regulate model complexity while accurately capturing natural imaging data. What's more, the variety of potential models phrased in terms of submodules is not limited to just shifting with circular boundary conditions. In other work by our group, we have found an entire family of tensor-tensor products that one might use to replace the t-product \cite{Kernfeld_HPAT}, some of which are also well suited to imaging tasks. For example, we have a tube-scalar product that replaces convolution with periodic boundary conditions (implicit in the t-product definition) with convolutions using reflected boundary conditions; this has been shown to improve image de-blurring algorithms \cite{Ng00cosinetransform}. The methods presented in this paper will be easily extended to use those new adaptations.

\section{Sparse Submodule Clustering}

Our algorithm and the development below is based on the following algebraic assumption. We assume that the data, viewed as a list of oriented matrices, comes from a \emph{union of disjoint free submodules}. The task is to find these submodules and group the data into their respective clusters. 

\subsection{The Algorithm}
\label{sec:SSmC_Alg}

To develop our algorithm, we will need to define the following tensor norms.


\begin{definition}  By $\|\TA\|_{F1}$, we mean $\sum_{i,j} ||\TA(i,j,:)||_F$, a group LASSO type norm equal to the sum of the Frobenius norms (F-norms) of the tubes. \end{definition}

\begin{definition}  $||\TA||_{FF1}$ denotes $\sum_{i} ||\TA(i,:,:)||_F$, a group LASSO norm equal to the sum of the F-norms of the horizontal slices (rows).
\end{definition}
Norms like these have been used to select for relevant groups of factors in regression problems \cite{RSSB:RSSB532}, and they tend to force the terms inside the F-norms to survive as a group if the corresponding multiplicands contribute enough to the model or to be driven to zero as a group otherwise. Here, we hope to select relevant samples for the reconstruction of each cluster. 

Let $\T{Y}$ denote the data tensor of size $H \times N \times D$ obtained by arranging the $H \times D$ images (data points) laterally as oriented matrices  $\ormat{Y}_n, \, n= 1,2,...,N$. Then the algorithm is based on the following principle. If we seek a sparse t-linear representation of an oriented matrix (image as a data point), then ideally only the matrices from the same submodule will contribute. In other words, samples from a submodule will provide an efficient generating set from the submodule.  The idea is shown in Figure \ref{fig:SSmC1}. 

This characterization allows us build an ``affinity tensor" $\T{W}$ by solving for the following optimization problem.
\begin{align}
\T{W} & =\arg \min_{\T{C}} \|\T{C}\|_{F1}+\lambda_h \| \T{C} \|_{FF1} + \lambda_g \|\T{Y}-\T{Y}\ast\T{C}\|_{F}^2 \nonumber \\
&  \text{ s.t. } \T{C}_{iik}=0, \,\, i,k = 1,2,...,N
\end{align}

For affine submodules \footnote{For our purposes, an affine submodule is a submodule translated away from the origin.}, we add the constraint $\sum_{i} \T{W}(i,j,:)= \V{e}_1$
Because the latter term is the identity element under convolution, this constraint gives the following property (proven below). Suppose the algorithm  discovers a self-representation $\ormat{X} = \sum_{i = 1 }^ N \tube{c}_i \ormat{Y}_i$ where $\tube{c}_i =0$ unless  $\ormat{Y}_i$  and $\ormat{X}$ share a cluster. Suppose the  whole cluster is translated by $\ormat{M}$ . Then, the constraint means the same set of coefficients gives an efficient representation for $\ormat{X}+\ormat{M}$ in terms of $\ormat{Y}_i + \ormat{M}$, because 
\begin{align*}
 \sum_{i} (\ormat{Y}_i+\ormat{M})\ast\T{W}_{ij}
&=\sum_{i} \ormat{Y}_i\ast\T{W}_{ij}+\sum_{i} \ormat{M}\ast\T{W}_{ij} =\sum_{i} \ormat{Y}_i\ast\T{W}_{ij}+\ormat{M}
\end{align*}
where for compactness, $\T{W}_{ij}=\T{W}(i,j,:)$. Submodules will then have a good chance of a successful self-representation for each datum, provided the coefficients from different clusters remain zero. The intuition is that hopefully the probability of the ÒbadÓ coefficients remaining zero will then increase.


\begin{figure}[htbp]
\centering 
  \includegraphics[height= 1.0 in]{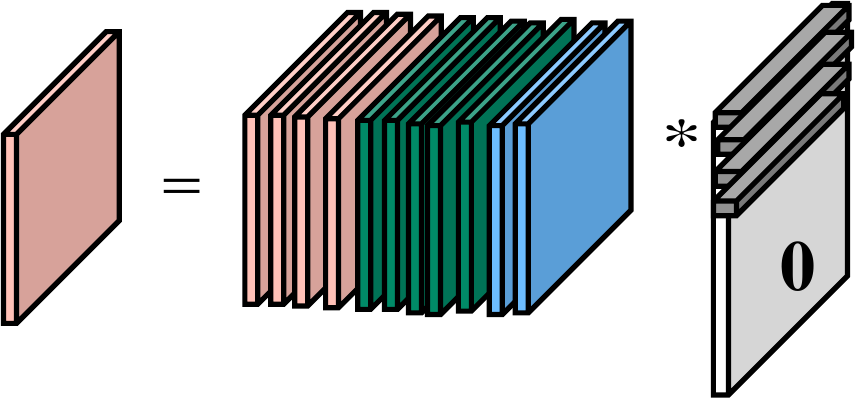}
  \caption{ \footnotesize{Sparsity in self-representation via t-product.} }
\label{fig:SSmC1}
\end{figure} 

Because of convolution-multiplication properties, the computations can be carried out most efficiently in the Fourier domain. Viewed in Fourier space, the equivalent task is to solve for
\begin{align}
&\hat{\T{W}} =\argmin_{\T{C}} \|\hat{\T{C}}\|_{F1}+\lambda_h\|\hat{\T{C}}\|_{FF1}+\lambda_g \|\hat{\T{Y}}-\hat{\T{Y}}\odot \hat{\T{C}}\|_{F}^2 \nonumber \\
& s. t. \,\,\, \hat{\T{C}}_{iik}=0\,\,\,; \sum_{i}\hat{\T{C}}_{ij}=[1,1,\cdots, 1].
\end{align}
Here, $\T{A}\odot \T{B}$ is the array resulting from face-wise matrix multiplication (see Section \ref{sec:tprod_fft}, notation).

Next, we build the affinity matrix \M{M} by taking tube-wise Frobenius norms of the affinity tensor and symmetricizing: 
\begin{equation}
\M{M}(i,j)=||\TW(j,i,:)||_F+||\TW(i,j,:)||_F.
\end{equation}

Finally, we run spectral clustering on the matrix $\M{M}$, inserting it at step two of the procedure on page two of \cite{njwspectral}. We chose to use the same clustering method used in \cite{Elhamifar:0zr}. It was proposed in \cite{njwspectral} and the version we use was implemented by the authors of \cite{verma}.

Our proposed algorithm is summarized by the pseudocode in Algorithm \ref{alg:SSmC}.
\begin{figure}[htbp]
    \begin{minipage}{\textwidth}
\begin{algorithm}[H]
\caption{Sparse sub-module clustering}
\label{alg:SSmC}
\begin{algorithmic}
\STATE \textbf{Input}: Data arranged into a 3-D tensor $\T{Y}$, number of clusters $k$.\\
\STATE 1. Solve the sparse self-representation optimization problem 
\begin{align*}
\T{W} & =\arg \min_{\T{C}} \|\T{C}\|_{F1}+\lambda_h \| \T{C} \|_{FF1} + \lambda_g \|\T{Y}-\T{Y}\ast\T{C}\|_{F}^2 \nonumber \\
&  \text{ s.t. } \T{C}_{iik}=0, \,\, i,k = 1,2,...,N
\end{align*}
 
\STATE 2. Find the affinity matrix $\M{M} \in \Real^{N \times N}$ using 
\begin{equation*}
\M{M}(i,j)=||\TW(j,i,:)||_F+||\TW(i,j,:)||_F.
\end{equation*}

\STATE 3. Apply spectral clustering (version from  \cite{verma}) to $\M{M}$. \\

\STATE \textbf{Output}: Obtain clusters from STEP 3.

\end{algorithmic}
\end{algorithm}
\end{minipage}
\end{figure}

\subsection{Performance Guarantees}
\label{sec:SSmC_Theorems}

We now develop three results to give conditions under which the sparse representation of a point will use only others from its own submodule. These results are analogous to the results on Sparse Subspace Clustering in \cite{Elhamifar:0zr}, which indicates the richness of proposed framework in its ability to borrow intuition from the traditional vector space setting.

The first theorem requires a strict condition before it can apply: the submodules must be linearly independent. The second and third establish conditions under which submodules may be merely disjoint. Note that linear independence is a special case of disjointness, as the union is contained within the sum. The peril with disjointness is that a t-linear combination of points from a pair of submodules $\mathbb{S}_1$ and $\mathbb{S}_2$ may no longer lie within $\mathbb{S}_1 \cup \mathbb{S}_2$, so disjointness may allow elements of $\mathbb{S}_3$ to be manufactured out of elements from $\mathbb{S}_1$ and $\mathbb{S}_2$.

For the case where $\mathbb{S}_i\cap\sum_{j\neq 1}^J \mathbb{S}_j =0$, i.e. independent submodules, the following result holds. The symbol $\sum$ indicates a sum of submodules as in Definition \ref{sum}.
\begin{theorem}
\label{thm:1}
Consider oriented data matrices $\{\ormat{Y}_n\}_{n=1}^N$ in $\mathbb{K}_D^H$. Suppose the data are within submodules $\{\mathbb{S}_j \}_{j=1}^J$ with dimensions $\{d_j\}_{j=1}^J$. Given $\ormat{Y}\in \mathbb{S}_i \subset \mathbb{K}_D^H$, denote by $\TY_i$ a tensor in $\mathbb{K}_D^{H\times(N_i-1)}$ whose lateral slices contain all data from $\mathbb{S}_i$ except $\TY$ itself. Denote by $\TY_{-i}$ a tensor in $\mathbb{K}_D^{H\times N-N_i}$ containing the data from other submodules.  (Here, $N_i$ is the number of data points in $\mathbb{S}_i$).
If the submodules are independent and $\left[\begin{array}{c}\ormat{C}^\ast\\\ormat{C}_-^\ast\end{array}\right]=\text{argmin}_{{\mathbb{K}_d^{h\times N-1}}}\left\|\left[\begin{array}{c}\ormat{C}\\\ormat{C}_-\end{array}\right]\right\|_{F1}$ s.t. $\ormat{Y}=[\TY_i \TY_{-i}]\ast\left[\begin{array}{c}\ormat{C}\\\ormat{C}_-\end{array}\right]$ then $\ormat{C}_-^\ast=0$.
\end{theorem}
Note that $\ormat{Y}=\TY_i\ast \ormat{C}+\TY_{-i}\ast \ormat{C}_-$ by properties of the t-product. In other words, conformable block partitioning works the same way with the t-product that it does for matrices.

For cases where submodules are disjoint, but not linearly independent, we have the following result. 

\begin{theorem}
\label{thm:2}
Suppose the data are set up according to the hypotheses of the preceding theorem.
 We define a pair of auxiliary quantities: given $\ormat{X}$, 
\begin{equation*}
\ormat{A}_i=\text{argmin} ||\ormat{A}||_{F1}\ s.t.\ \TY_i\ast \ormat{A} =\ormat{X} \,\,\,\, \mbox{and}\,\,\,
\ormat{A}_{-i}=\text{argmin} ||\ormat{A}||_{F1}\ s.t.\ \TY_{-i}\ast \ormat{A} =\ormat{X} 
\end{equation*}

If the submodules are disjoint and $\forall \ormat{X}$, $||\ormat{A}_i||_{F1}<||\ormat{A}_{-i}||_{F1}$, then the minimum argument $\left[\begin{array}{c}\ormat{C}_i\\  \ormat{C}_{-i}\end{array}\right]=\text{argmin} \left\|\left[\begin{array}{c}\ormat{C}\\\ormat{C}_-\end{array}\right]\right\|_{F1}  s.t.\ \ormat{X}=\left[\begin{array}{cc}\TY_i &\TY_{-i}\end{array}\right]\left[\begin{array}{c}\ormat{C}\\\ormat{C}_-\end{array}\right]$ has a zero block: $\ormat{C}_{-i}=0$. Thus, only elements of $\mathbb{S}_i$ are used.
\end{theorem}

The performance guarantees are derived in terms of \emph{angles} between 2-D data and coherency between submodules. We introduce two definitions to relay these notions in our algebraic setting.

Let $\theta\in\mathbb{R}^{1\times 1\times D}$ represent a tubal angle as defined in \cite{Kilmer:2013kx} , i.e. $\mbox{\tt cos}(\theta(\ormat{A},\ormat{B}))=\frac{\ormat{A}^T\ast\ormat{B}+\ormat{B}^T\ast\ormat{A}}{2||\ormat{A}||_F||\ormat{B}||_F}$ for $\ormat{A}$ and $\ormat{B}$ oriented matrices in $\mathbb{K}^{H}_D$. Cosine here acts on each entry of the tube individually.

\begin{definition}For submodules $\{\mathbb{S}_j\}_{j=1}^J$, define the tubal-coherence of $\mathbb{S}_i$ and $\mathbb{S}_j$ to be  
\begin{equation}
c_{ij}=\max_{\ormat{V}_i\in \mathbb{S}_i, \ormat{V}_j\in \mathbb{S}_j}\{||\mbox{\rm{cos} }(\theta(\ormat{V}_i,\ormat{V}_j))||_F\}.
\end{equation}
\end{definition}
We now have the following lemma which describes an important linear algebraic property (whose proof can be found in the supplementary material).
\begin{lemma}
\label{lem:span}
Suppose $\TY_i\in\mathbb{K}_D^{H\times d_i}$ has lateral slices drawn from a submodule of (submodular) dimension $d_i$ and that none of the Fourier-domain frontal slices of $\TY_i$ have any zero singular values. Then, the lateral slices of $\TY_i$ form a generating set for the submodule. \label{span_lemma}
\end{lemma}
We can now state the main result on condition for exact submodule clustering.
\begin{theorem}
\label{thm:3}
Suppose the data are set up as in the preceding theorems.
Let $W_i$ denote the set of all size $H\times d_i\times D$, full-rank sub-tensors of $\TY_i$. In other words, every element of $W_i$ is a tensor in $\mathbb{K}_D^{H\times d_i}$ whose lateral slices are drawn from $\TY_i$, and an element of $W_i$, when viewed in the Fourier domain, has frontal faces with full rank.

For any $\ormat{X}\in  \mathbb{S}_i\cap [\sum_{j=1} \mathbb{S}_j]$, suppose the condition 
\begin{equation*}
\sqrt{d_i}\cdot \max_{i\neq j}  (c_{ij})\cdot \sigma_{\max}(\bcirc(\TY_{-i}))<\max_{\tilde{\TY_i}\in W_i}\{\sigma_{\min}(\bcirc(\tilde{\TY_i}))\}
\end{equation*}
holds, where $\sigma_{\min}(\cdot)$ and $\sigma_{\max}(\cdot)$ return the maximum and the minimum singular values respectively. Note that both will be nonzero due to the full-rank assumption. Then Theorem~\ref{thm:2} will apply, or in other words $||\ormat{A}_{i}||_{F1}<||\ormat{A}_{-i}||_{F1}$.

\end{theorem}

\section{Numerical Results}
\label{sec:Sim}

\textbf{Synthetic data} - In order to assess speed and reliability on modestly sized datasets, synthetic data were generated lying along multiple submodules. One set of test parameters was chosen to mimic, in size, a subset of the MNIST handwritten digit dataset. Each synthetic image is 28 by 28, with images distributed along 4 clusters. Total dataset size was varied from 40 to 200 ``images.'' The other set of test parameters was chosen to be larger, on the scale of a face database. Each synthetic image is 200 by 200, with images distributed along 3 clusters. Total dataset size was varied from 30 to 150 ``images.'' Runtime in seconds for both tests is displayed in Table \ref{fig:synth_runtimes}. In our tests of SSmC, misclassification rate was zero for all synthetic data. 

\begin{figure}[h!]
\centering
\begin{tabular}{|l|c|c|c|c|c|}
  \hline
  MNIST & 16.2581 & 60.4360 & 290.1172 & 731.5302 &  539.5757  \\ \hline
  Face &    
74.1688 & 139.7167 & 231.1336 & 675.3414 & 2868.800 \\ \hline
\end{tabular}
\caption{\footnotesize{Run times for the SSmC algorithm implemented using Alternating Direction Method of Multipliers (ADMM) \cite{Boyd_ADMM}.\label{fig:synth_runtimes}}}
\vspace{-3mm}
\end{figure}



\textbf{Real data} - All of the tests in this section use an implementation of SSC contained as a special case of our implementation of SSmC. Our implementation lacks one portion of the SSC algorithm as presented in \cite{Elhamifar:0zr}: we do not make a provision for sparse outlying entries, which \cite{Elhamifar:0zr} does by solving this problem.
 \begin{align}
\M{W} & =\arg \min_{\M{C}} \|\M{C}\|_{F1}+\lambda_h \| \M{C} \|_{FF1}+\lambda_z \|\M{E}\|_{1} + \lambda_g \|\M{Y}-\M{Y}\M{C}+\M{E}\|_{F}^2 \nonumber \\
&  \text{ s.t. } \M{C}_{ii}=0, \,\, i,k = 1,2,...,N
\end{align}
The difference of note is the matrix $\M{E}$, penalized with an $\ell_1$ norm to promote sparsity. To be clear, we could incorporate this extra term into SSmC and SSC, but this more detailed comparison is left for future work. 

The t-product has performed well in the past as a tool for face recognition \cite{Hao:2013uq}. It is also known that images of the same object under various lighting conditions approximately form a low-dimensional subspace \cite{Basri:2001uq}, \cite{Ramamoorthi}. So in the following we test the performance of the method on clustering images from various Data Bases (DB).

\textbf{(1) Weizmann Face Data Base \cite{Weizman}}- In this experiment the aim is to group together faces regardless of lighting condition (same pose). Setting $\lambda_h =0$, we selected $\lambda_g$ with a training set of four faces at four lighting conditions, monitoring quality via a heuristic. Using those values, a test set of 28 faces at four illuminations was successfully grouped using SSmC and SSC into twenty-eight segments by SSmC, each containing only one person (no corresponding figure). 

In another series of tests with no training stage, we find that SSmC displays robustness to the choice of $\lambda_g$. This time, 36 images were used: four people, each at nine lighting conditions. SSC succeeded only in a narrow range with $\lambda_g$ between $10^{-6} $ and $10^{-7}$, while SSmC succeeds with $\lambda_g$ between $10^{-4} $ and $10^{-8}$. While this may not be a useful end in itself, we take it as a sign that a \emph{union of submodules is a reasonable model for a database of this form and that the added complexity of the t-product may be warranted}. Furthermore, SSmC can withstand a level of additive noise that foils SSC on the 36 images in the preceding paragraph. Pixel values ranged from zero to 255, and isotropic Gaussian noise of standard deviation 20 was added. SSmC succeeded for $\lambda_g$ between $10^{-7.5} $ and $10^{-8}$. All trials were run with SSC using identical parameters and noise instantiations, and SSC was unable to cluster the faces.

\begin{figure}[htbp]
\centering 
\vspace{-2.5cm}
  \includegraphics[height= 4 in]{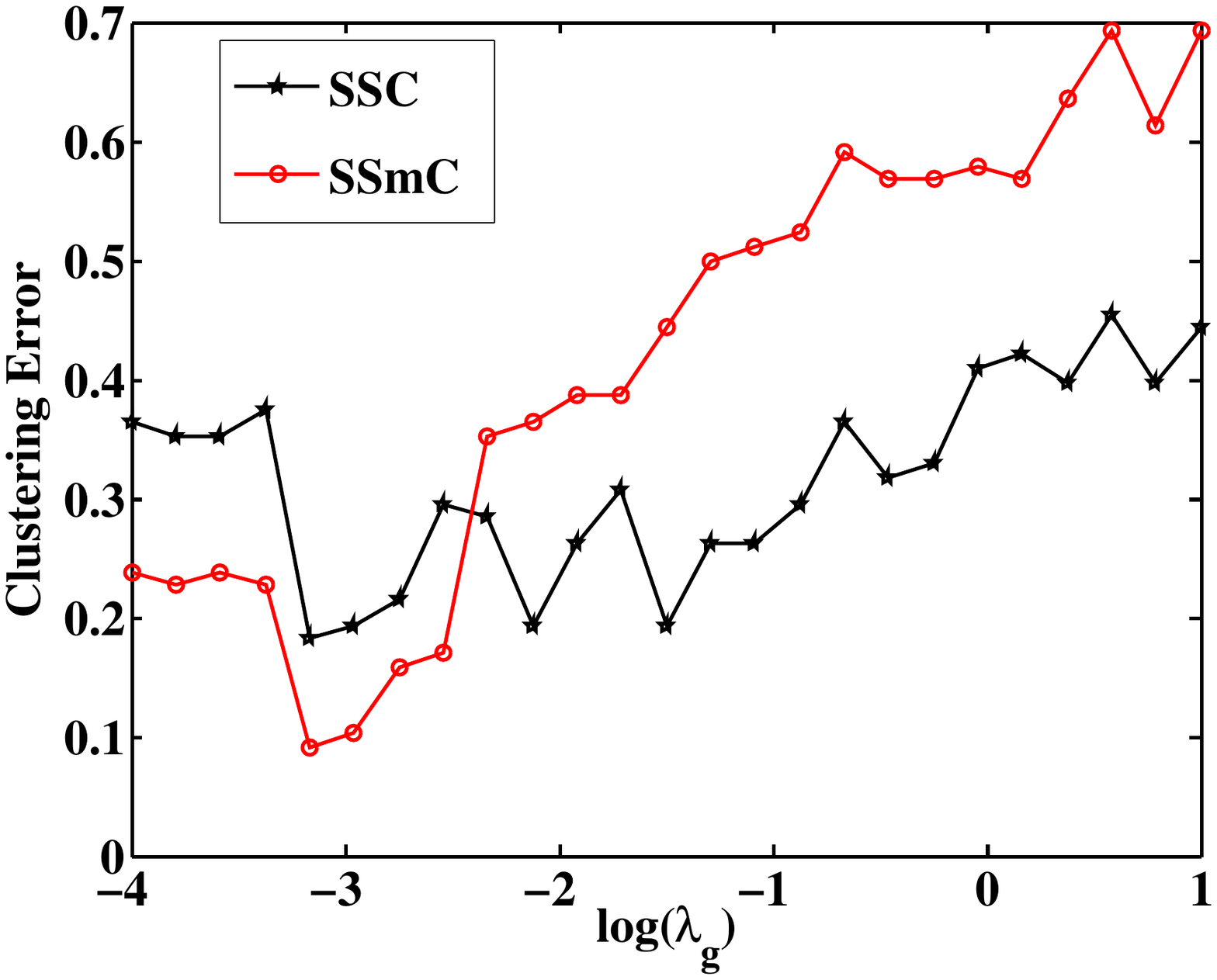}
  \vspace{-2cm}
  \caption{ \footnotesize{Clustering error performance on original Yale DB.} }
\label{fig:SSmC_SSC_Yale}
\end{figure} \textbf{(2) Yale Face Data Bases \cite{Yale_DB1,Yale_DB2}} - We tested the approach on two Yale Data Bases, original \cite{Yale_DB1} and extended \cite{Yale_DB2}.  The original data base contains contains 165 grayscale images of 15 individuals. There are 11 images per subject, one per different facial expression or configuration: center-light, w/glasses, happy, left-light, w/no glasses, normal, right-light, sad, sleepy, surprised, and wink. To reduce the processing time, we downsampled images from the original size of $240\times 320$ to $81\times 107$ and we picked the images corresponding to the first 8 subjects. \emph{No other preprocessing in terms of centering, rotating etc. was performed}. Figure \ref{fig:SSmC_SSC_Yale} illustrates the clustering error for various values of $\lambda_g$ for SSmC Vs SSC. Note that at optimal $\lambda_g$ the clustering error for SSmC is $0.09$ whereas clustering error for SSC is $0.18$. 

\begin{figure}[htbp]
\centering 
\vspace{-2.5cm}
  \includegraphics[height= 4 in]{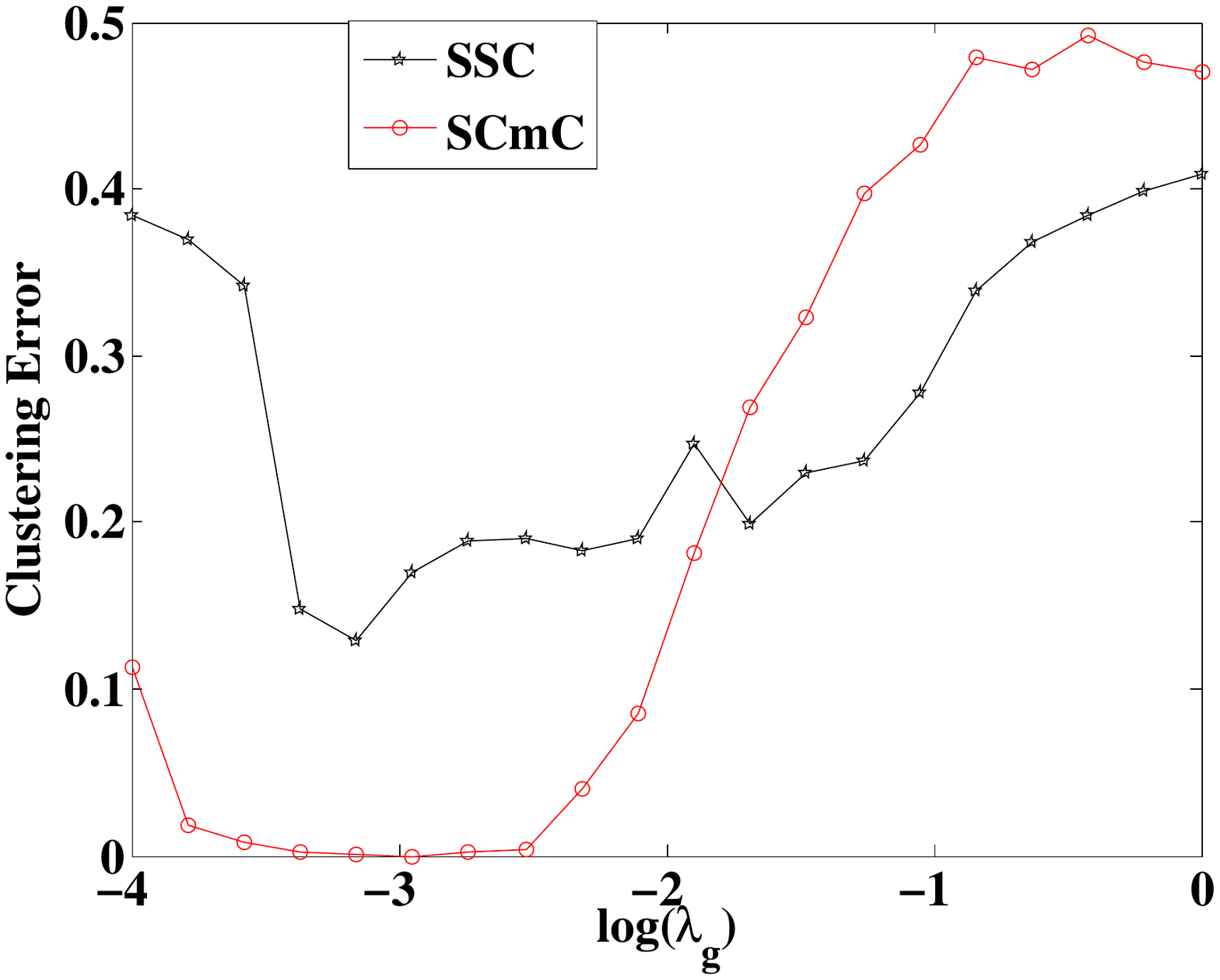}
  \vspace{-2.5cm}
  \caption{Clustering error performance on Extended Yale DB. }
\label{fig:SSmC_SSC_Yale_Ext}
\end{figure} 
We next performed experiments with Extended Yale Data Base. For this experiment we again reduced the size of the images from  $480 \times 640$ to $120 \times 111$ by first downsampling by a factor of $4$ along each axis and then cropping the columns (in order to approximately focus on the faces) by only keeping the indices from $30:140$. We normalize the pixel intensity to lie in $[0,1]$. We then took subject numbers $11,12,15,16,17$ and randomly picked 25 images per person from the $65\times8 = 520$ images taken under various poses and different illuminations. Figure \ref{fig:SSmC_SSC_Yale_Ext} shows the clustering performance of SSmC vs SSC averaged over 20 trials. The SSmC was able to achieve perfect clustering while SSC was not. We believe that this is due to the shift invariant nature of the SSmC method, which requires less preprocessing. This is further borne out by the following experiment on the MNIST handwritten character clustering problem.

\begin{figure}[htbp]
\centering 
\vspace{-2.5cm}
  \includegraphics[height= 4 in]{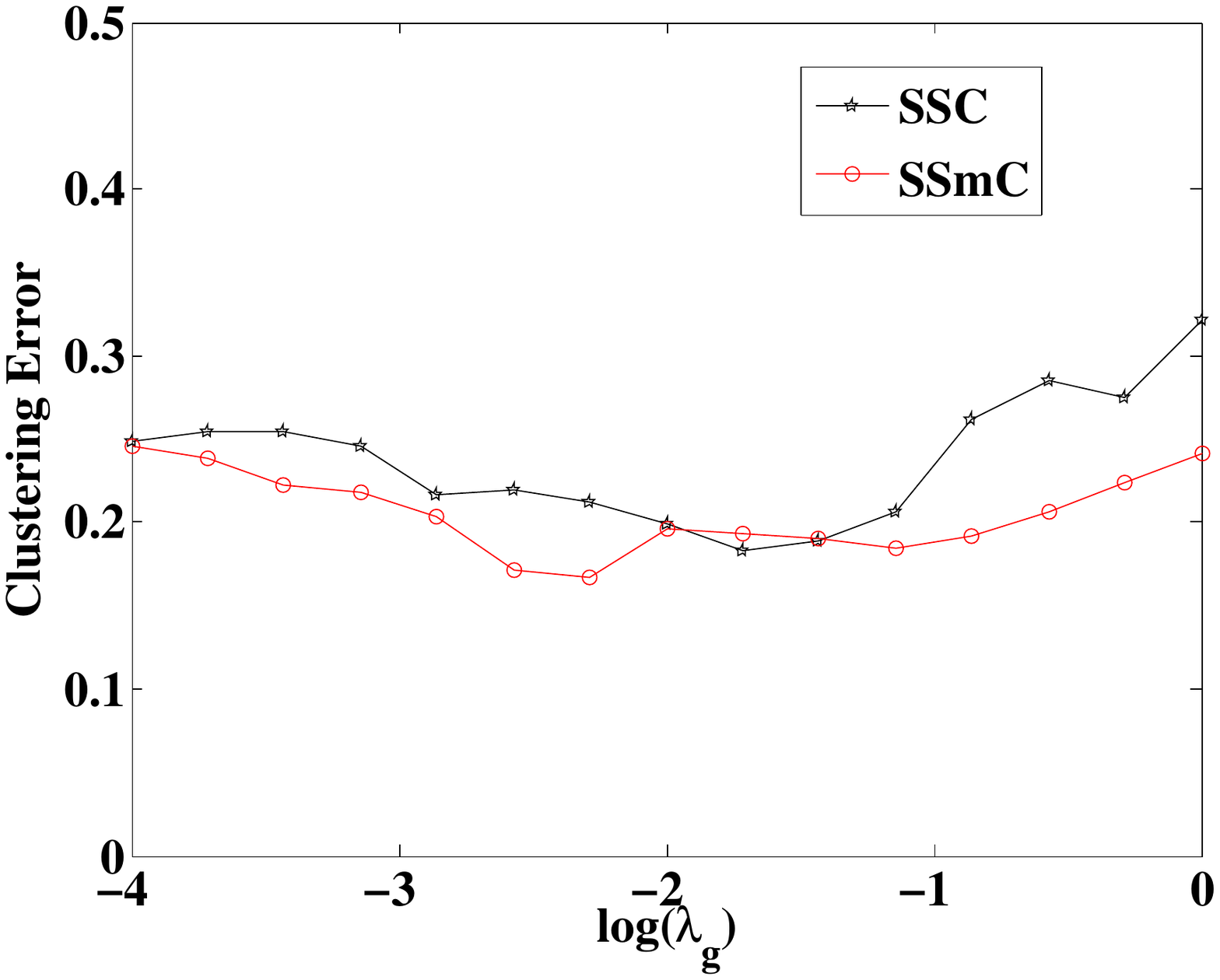}
  \vspace{-2.5 cm}
  \caption{ Clustering error performance on MNIST DB. }
  \vspace{0cm}
\label{fig:SSmC_SSC_MNIST}
\end{figure} \textbf{(3) MNIST handwritten Data Base \cite{MNIST}} - In this experiment we seek to cluster digits $2,4,8$ by randomly taking $20$ instances for each digit from the labeled data. The clustering error performance curve, averaged over $25$ such random trials, is shown in Figure \ref{fig:SSmC_SSC_MNIST}.  It turns out that for this data set SSC remains competitive to SSmC, with SSmC being slightly better at optimal $\lambda_g$. 

In order to check the shift invariance property of SSmC we randomly shifted the digits horizontally, uniformly with respect to the center by $6$ pixels either side. In this case for the best value of $\lambda_g$ SSmC exhibited a clustering error of $0.2$ while clustering error for SSC remained at $0.5$ for all values $\lambda_g$. This implies that \emph{SSmC is more robust to processing data without needing to preprocess for centering and alignment}. 

\section{Conclusions and Future work}

We presented a novel method to cluster 2-D data while preserving the multi-way aspect of the data. Initial results show robustness compared to existing approaches. In future we plan to carry out experiments with 3-D data sets, such as the American Sign Language Lexicon Video Dataset (ASLLVD). On the computational front, the current implementation of SSmC does not use multiple processors, and it is not suited to large datasets, as it takes minutes to run on the test examples. However, we believe the algorithm can be effectively parallelized. During this project, the paper \cite{HeckelB13_ISIT} proposed a fast alternative to SSC, and their work could be extended to quickly perform submodule clustering as well.

\section{Appendix}
\subsection{Proof of Theorem \ref{thm:1}}

\begin{proof}
If $\ormat{C}_-^\ast\neq0$, then $\ormat{Y}=\TY_i\ast \ormat{C}^\ast+\TY_{-i}\ast \ormat{C}_-^\ast\Rightarrow \TY_{-i}\ast \ormat{C}_-^\ast=\TY_{i}\ast \ormat{C}^\ast+\ormat{Y}\in \mathbb{S}_i$, since the module generated by the lateral slices of $\TY_{-i}$ is in $\oplus_{i\neq j} \mathbb{S}_j$, it must be the case that $\TY_{-i}\ast \ormat{C}_-^\ast\in \oplus_{i\neq j} \mathbb{S}_j$, also, and so $\TY_{-i}\ast \ormat{C}_-^\ast$ is in the intersection and must be zero. Thus $\ormat{Y}=\TY_i\ast \ormat{C}^\ast+0$, and $\ormat{C}_-^\ast$ may be replaced by zero. This lowers the F1 norm, contradicting the assumption that $\left[\begin{array}{c}\ormat{C}^\ast\\\ormat{C}_-^\ast\end{array}\right]$ is optimal.
\end{proof}

\subsection{Proof of Theorem \ref{thm:2}}

\begin{proof}
Suppose $\ormat{C}_{-i}\neq0$. Using the hypotheses of the theorem, we will derive a contradiction by constructing better coefficients. Let $\ormat{X}=\TY_{-i}\ast \ormat{C}_{-i}=\ormat{Y}-\TY_i\ast \ormat{C}_i$. Then $\TY_i\ast \ormat{A}_i=\ormat{X}=\TY_{-i}\ast \ormat{A}_{-i}=\ormat{Y}-\TY_i\ast \ormat{C}_i$, so 
\begin{align}
\ormat{Y} & =\left[\begin{array}{cc}\TY_i &\TY_{-i}\end{array}\right]\left[\begin{array}{c}\ormat{A}_i+\ormat{C}_i^\ast\\0\end{array}\right]\\
& =\left[\begin{array}{cc}\TY_i &\TY_{-i}\end{array}\right]\left[\begin{array}{c}\ormat{C}_i^\ast\\\ormat{A}_{-i}\end{array}\right]=\left[\begin{array}{cc}\TY_i &\TY_{-i}\end{array}\right]\left[\begin{array}{c}\ormat{C}_i\\\ormat{C}_{-i}\end{array}\right]
\end{align}
Also, $\TY_{i}\ast \ormat{A}_{-i}=\TY_{-i}\ast \ormat{C}_{-i}$, so $||\ormat{A}_{-i}||_{F1}<||\ormat{C}_{-i}||_{F1}$ by optimality of $\ormat{A}_{-i}$. Using the fact that $\left\|\left[\begin{array}{c}\ormat{V}_1\\\ormat{V}_2\end{array}\right]\right\|_{F1}=||\ormat{V}_1||_{F1}+||\ormat{V}_2||_{F1}$, 
\begin{align*}
\left\|\left[\begin{array}{c}\ormat{A}_i+\ormat{C}_i\\ 0\end{array}\right]\right\|_{F1}
&\leq ||\ormat{A}_i||_{F1}+||\ormat{C}_i||_{F1}\\
&<||\ormat{A}_{-i}||_{F1}+||\ormat{C}_i||_{F1}\\
&\leq ||\ormat{C}_{-i}||_{F1}+||\ormat{C}_i||_{F1}\\
&=\left\|\left[\begin{array}{c}\ormat{C}_i\\\ormat{C}_{-i}\end{array}\right]\right\|_{F1}
\end{align*}

Since $\left[\begin{array}{c}\ormat{C}_i\\\ormat{C}_{-i}\end{array}\right]$ was supposed to be optimal, we have reached a contradiction. 
\end{proof}

\subsection{Proof of Theorem \ref{thm:3}}

\begin{proof}
Theorem 2 applies if $||\ormat{A}_i||_{F1}<||\ormat{A}_{-i}||_{F1}$ (these are defined in relation to $\ormat{X}$ as before). This proof will derive $\beta_i$ and $\beta_{-i}$ such that $||\ormat{A}_i||_{F1}\leq \beta_i<\beta_{-i}\leq ||\ormat{A}||_{F1}$ with the middle statement $\beta_i<\beta_{-i}$ holding exactly when the condition listed in the hypotheses holds.


To find $\beta_i$, note that since $\tilde{\TY}_i$ is an element of $W_i$, its columns are in $\mathbb{S}_i$, so we can invoke Lemma \ref{span_lemma} to claim that $\ormat{X}\in  \mathbb{S}_i$ implies $\ormat{X}=\tilde{\TY}_i\ast\tilde{\ormat{A}_i}$ for some $\tilde{\ormat{A}}_i$. This new variable $\tilde{\ormat{A}_i}$ mimics ${\ormat{A}_i}$ in terms of its role, but it may be suboptimal in terms of F1 norm. There may be many eligible candidates, and we choose among them as follows.

Since $\ormat{X}=\tilde{\TY}_i\ast\tilde{\ormat{A}_i}$, can be equivalently written as 
$$\ormat{X}=\bcirc(\tilde{\TY}_i) {\tt unfold}(\tilde{\ormat{A}_i}) \,\, ,$$
Using Lemma \ref{lem:l3} (in the reverse direction) we have, $$ \|\tilde{\ormat{A}}\|_{F} \leq \sigma_{\min}( \bcirc(\tilde{\TY}_i)) \|\ormat{X}\|_{F}  $$



Now we may pad the $d_i\times1\times D$ array $\tilde{\ormat{A}}_i$ with zeros to form an $N_i\times1\times D$ array, $\ormat{A}_i^{+0}$, placing the zeroes to hit elements of ${\TY_i}$ that aren't present in the subtensor $\tilde{\TY_i}$. Then, $\tilde{\TY_i}\ast \tilde{\ormat{A}}_i=\TY_i\ast \tilde{\ormat{A}}_i^{+0}$, and we observe that $|| \ormat{A}_i^{+0}||_{F1}=||\ormat{A}_i ||_{F1}$ and $\TY_i\ast \tilde{\ormat{A}}_i^{+0}=
\ormat{X}=\TY_i\ast \ormat{A}_i$. By the F1-optimality of $\ormat{A}_i$, $||\ormat{A}_i||_{F1}\leq||\tilde{\ormat{A}}_i^{+0}||_{F1}$, and in the following we also use relationship between F- and F1-norms: $||\tilde{\ormat{A}}_i||_{F1}\leq\sqrt{d_i}||\tilde{\ormat{A}}_i||_F$. This can be shown by taking the F-norms of the tubes first and then applying the well known inequality $||v||_{1}\leq\sqrt{d_i}||v||_F$. Using these, we can find a useful bound for one of the coefficient norms.
\begin{align}
||\ormat{A}_i||_{F1}& \leq ||\tilde{\ormat{A}}_i^{+0}||_{F1}=||\tilde{\ormat{A}}_i ||_{F1}\leq \sqrt{d_i}|| \tilde{\ormat{A}}_i||_F \nonumber \\
& \leq \frac{\sqrt{d_i}}{\sigma_{\min}(\bcirc(\tilde{\TY}_i))}||\ormat{X}||_F=\beta(\tilde{\TY}_i)
\end{align}
To push the bound down, we may take $\beta_{i}=\min_{\tilde{\TY_i}\in W_i} \beta(\tilde{\TY}_i)$. 

To bound $||\ormat{A}_{-i}||_{F1}$ below, we make use of lemmas \ref{submultiplicative} (in the first inequality) and \ref{opnorm} (in the fourth line). The largest singular value fourier coefficient of $\TY_{-i}$ is  labeled $\hat{\sigma}_1^{(\max)}(\TY_{-i})$.

\begin{align*}
||\ormat{X}||_{F}^2&\leq ||\ormat{X}^{T}\ast\ormat{X}||_F \nonumber \\
& =||\ormat{A}_{-i}^T\ast \TY_{-i}^T\ast \ormat{X}+\ormat{X}^T\ast\TY_{-i}\ast \ormat{A}_{-i}||_F/2\\
&=||\mbox{\tt cos}(\theta(\ormat{X}, \TY_{-i}\ast \ormat{A}_{-i}))||_F \cdot||\TY_{-i}\ast \ormat{A}_{-i}||_F||\ormat{X}||_F\\
&\leq \max_{i\neq j} c_{ij}\cdot \sigma_{\max}(\bcirc(\TY_{-i}))\cdot || \ormat{A}_{-i}||_F||\ormat{X}||_F\\
&\leq \max_{i\neq j} c_{ij} \cdot\sigma_{\max}(\bcirc(\TY_{-i}))\cdot || \ormat{A}_{-i}||_{F1}||\ormat{X}||_F
\end{align*}

The inequality 
\begin{equation*}
||\ormat{X}||_{F}^2\leq \max_{i\neq j}c_{ij} \cdot \sigma_{\max}(\bcirc(\TY_{-i}))\cdot ||\ormat{A}_{-i}||_{F1}||\ormat{X}||_F
\end{equation*}
can be rewritten as 
\begin{equation*}
\frac{||\ormat{X}||_F}{ c_{ij}\cdot \sigma_{\max}(\bcirc(\TY_{-i}))}\leq || \ormat{A}_{-i}||_{F1}
\end{equation*}
so if $$\beta_{-i}=\frac{||\ormat{X}||_F}{ \max_{i\neq j}c_{ij}\cdot \sigma_{\max}(\bcirc(\TY_{-i}))}\,\, ,$$ then $||\ormat{A}_{i}||_{F1}\leq \beta_{i}<\beta_{-i}\leq ||\ormat{A}_{-i}||_{F1}$ exactly when the condition in the hypotheses holds.
\end{proof}

\subsection{Lemmas Used in Proof of Theorem \ref{thm:3}}
We prove the three lemmas here, stating those not present earlier. 


\begin{lemma}
\label{lem:l2}
For any oriented matrix $\|\ormat{X}\|_{F}^{2} \leq \| \ormat{X} \ast \ormat{X}\|_{F}$.\label{submultiplicative}
\end{lemma}

\begin{proof}
\begin{align}
||\ormat{X}||_F^2 &= \frac{1}{N_3}||\hat{\ormat{X}}||_F^2 \text{ fft increases norm by $\sqrt{N_3}$; divide by $N_3$ to undo }\\
&=\frac{1}{N_3}\sum_i||\hat{\ormat{X}}^{(i)}||_F^2\\
&=\frac{1}{N_3}\sum_i \hat{\ormat{X}}^{(i)\top}\hat{\ormat{X}}^{(i)}\\
&=\frac{1}{N_3}||\widehat{\ormat{X}^\top*\ormat{X}}||_1 \text{ sum is L1 norm because of positive entries }\\
&\leq \frac{\sqrt{N_3}}{N_3}||\widehat{\ormat{X}^\top*\ormat{X}}||_2 \text{ L1-L2 inequality }\\
&\leq ||\ormat{X}^\top*\ormat{X}||_2\text{ ifft decreases norm by $\sqrt{N_3}$; multiply by $\sqrt{N_3}$ to compensate}\ 
\end{align}

\end{proof}

%
%

%

\begin{lemma}
\label{lem:l3}
We have the following inequality, 
$||\TY_{-i}\ast \ormat{A}_{-i}||_F \leq \sigma_{\max}(\bcirc(\TY_{-i})) \cdot ||\ormat{A}_{-i}||_F$ where $\sigma_{max}(\cdot)$ denotes the maximum singular value of the matrix in the argument. \label{opnorm}
\end{lemma}

\begin{proof}
The result follows by noting that $$\TY_{-i} * \ormat{A}_{-i} = \Fold{\bcirc(\TY_{-i})\Unfold{\ormat{A}_{-i}}}$$
\end{proof}

Proof of Lemma \ref{lem:span}
\begin{proof}
We first restate the conclusion: if the $H$ by $d_i$ by $D$ tensor $\TY$ has lateral slices drawn from a free submodule $\mathbb{S}_i$ of dimension $d_i$ and all the singular values of its Fourier-domain frontal slices are nonzero, then for any element $\ormat{X}$ of $\mathbb{S}_i$, we must show the existence of a 1 by $d_i$ by $D$ coefficient tensor $\ormat{C}$ so that $\ormat{X}$ is the t-product of $\TY$ and $\ormat{C}$. Here $\mathbb{S}_i$ consists of $H$ by $W$ by $D$ tensors, with $H>d_i$.

We prove the lemma by contradiction. We first show that the module $\mathbb{S}_i$ has dimension $D\times d_i$ when considered as a vector space over the complex numbers. Then, we suppose that some $\ormat{X}$ violates the claim, and this has the consequence of increasing the dimension.  

To establish the dimension of $\mathbb{S}_i$, we consider linearly independent generating elements $\{\ormat{V}_j\}$ of $\mathbb{S}_i$, concatenated into a tensor $\TV$. Such a generating set exists by the definition of a free module. The submodule may be decomposed into a direct sum of $D$ vector spaces each of whose Fourier-domain representations has elements with zeroes at every frontal slice but one. Here, $\oplus$ is a traditional sum of vector spaces, obtained by concatenating generating sets.

$\mathbb{S}_i=\sum_{k=1}^D [\mathbb{S}_i\cap \{T\in\mathbb{C}^{H\times W\times D}| j\neq k\Rightarrow \text{fft}(T)(:,:,j)=0\}]$
 
Because the submodule $\mathbb{S}_i$ is generated via the T-product, �linear combinations� of $\TV$'s columns (lateral slices) can be represented simply in the Fourier domain as face-by-face matrix multiplication with the Fourier representation of $\TV$ interacting with coefficient vectors at each face. 

$\ormat{X}\in \mathbb{S}_i \Rightarrow \exists \ormat{C} s.t. \text{fft}(\ormat{X})^{(n)}=\text{fft}(\TV)^{(n)}\cdot \text{fft}(\ormat{C})^{(n)}$
 
At any given face, the $N$ by $d_i$ matrix containing Fourier coefficients of $\TV$ has no zero singular values, because otherwise $\TV$ would have a nontrivial null space and $\{\ormat{V}_j\}$ would not be linearly independent. Thus, the columns of that matrix are linearly independent and the subspace at that face contributes $d_i$ degrees of freedom to the direct sum. With a total of $D$ faces, the result is a complex vector space of dimension $D\times d_i$.

Suppose, for the sake of contradiction, that some $\ormat{X}$ violates the original claim. Then, for at least one face of the Fourier representation, there is no coefficient tensor $\ormat{C}$ for which $\text{fft}(\ormat{X})^{(n)}=\text{fft}(\TV)^{(n)}\cdot \text{fft}(\ormat{C})^{(n)}$. At that face, the augmented matrix $\left[\text{fft}(\ormat{X})^{(n)}|\text{fft}(\TV)^{(n)}\right]$ must  have linearly independent columns. If it did not, there would be coefficients so that $\text{fft}(\ormat{X})^{(n)}\cdot \text{fft}(c_x)^{(n)}+\text{fft}(\TV)^{(n)}\cdot \text{fft}(\ormat{C})^{(n)}=0$, and we could divide by $-\text{fft}(c_x)^{(n)}$ to produce coefficients such that $\text{fft}(\ormat{X})^{(n)}=\text{fft}(\TV)^{(n)}\cdot \text{fft}(\ormat{C})^{(n)}$. This is impossible, so we deduce that either no $\ormat{X}$ violates the original claim or $-\text{fft}(c_x)^{(n)}$ is zero (preventing division). 

In the latter case, the fact that $\text{fft}(\TV)^{(n)}$ is full rank implies the rest of the coefficients are $0$, so the columns of the augmented matrix are indeed linearly independent. Thus, at that face, the direct summand of complex dimension $d_i$, $\left[\text{fft}(\mathbb{S}_i)\cap\{T\in\mathbb{C}^{H\times W\times D}| j\neq k\Rightarrow T(:,j)=0\}\right]$ contains the set \begin{equation*}
\{T\in\mathbb{C}^{H\times W\times D}| T(:,:,n)=\text{fft}(\ormat{X})^{(n)}\cdot \text{fft}(\tube{c}_x)^{(n)}
\end{equation*}
\begin{equation*}
+\text{fft}(\TV)^{(n)}\text{fft}(\ormat{C})^{(n)}\} \text{ and } j\neq k \Rightarrow T(:,:,j)=0\}
\end{equation*}
which has (complex) dimension $d_i+1$. This is a contradiction. We conclude that no $\mathcal{X}$ may violate the original claim.
\end{proof}

\bibliographystyle{plain}
\bibliography{bibtensor,SAbib,elmbib,ZZbib,SSmCbibliography}

\end{document}